\documentclass[journal]{IEEEtran}
\usepackage{cite}
\usepackage{amsmath,amssymb,amsfonts}
 
\usepackage{amsthm}
\usepackage{algorithmic}
\usepackage{graphicx,color}
\usepackage{textcomp}
\usepackage{xcolor}
\usepackage{tikz}
\usepackage[ruled,vlined]{algorithm2e}
\usepackage[compatibility=false]{caption}
\SetKwInput{KwInput}{Input}
\SetKwInput{KwOutput}{Output}
\SetKwInput{KwRet}{Return}

\usetikzlibrary{arrows.meta, positioning, shapes.geometric, bending}
\usepackage{soul}
\usepackage{cleveref}
\usepackage{subfig}
\usepackage{algorithm,algorithmic}
\usepackage{mathtools}
\usepackage{amsthm}
\def\BibTeX{{\rm B\kern-.05em{\sc i\kern-.025em b}\kern-.08em
    T\kern-.1667em\lower.7ex\hbox{E}\kern-.125emX}}
\AtBeginDocument{\definecolor{ojcolor}{cmyk}{0.93,0.59,0.15,0.02}}

\usepackage{xcolor}

\DeclarePairedDelimiter\norm{\lVert}{\rVert}
\DeclarePairedDelimiter{\ip}{\langle}{\rangle}

\def\authorrefmark#1{\ensuremath{^{\textbf{#1}}}}

\usepackage{thmtools}
\DeclarePairedDelimiter\bracks{\lbrack}{\rbrack}
  
\newcommand{\algname}{MHLJ}
\newcommand{\er}{Erd\H{o}s-R\'{e}nyi}
\newcommand{\mh}{Metropolis-Hastings }
\newcommand{\E}[1]{\mathbb{E}\bracks*{#1}}

\usepackage{thm-restate}
\declaretheorem[name=Definition]{definition}

\declaretheorem[style=remark]{remark}

\newtheorem{coro}{Corollary}

\newtheorem{thm}{Theorem}

\newtheorem{lem}{Lemma}

\begin{document}

\title{Decentralized  Learning via Random Walk with Jumps}

\author{Zonghong Liu\authorrefmark{1}, Student Member, IEEE, Matthew Dwyer\authorrefmark{2},\\ and Salim El Rouayheb\authorrefmark{1}, Member, IEEE
\thanks{A part of this work has been presented in \cite{liu2024entrapment}. This work was supported in part by the Army Research Lab (ARL) under
 Grant W911NF-24-2-0172 and the National Science Foundation (NSF) under
 Grant CNS-2148182.}
\thanks{\authorrefmark{1}Zonghong Liu and Salim El Rouayheb are with the Department of Electrical and Computer Engineering, Rutgers University, Piscataway, NJ, USA (e-mail: zonghong.liu@rutgers.edu;  salim.elrouayheb@rutgers.edu).}
\thanks{\authorrefmark{2}Matthew Dwyer is with the DEVCOM Army Research Laboratory, Adelphi, MD, USA (e-mail: matthew.r.dwyer7.civ@army.mil).}
}

\maketitle
\begin{abstract}
We study decentralized learning over networks where data are distributed across nodes without a central coordinator. Random-walk learning is a token-based approach in which a single model is propagated across the network and updated at each visited node using local data, thereby incurring low communication and computational overheads. In weighted random-walk learning, the transition matrix is designed to achieve a desired sampling distribution, thereby speeding up convergence under data heterogeneity. We show that implementing weighted sampling via the Metropolis–Hastings algorithm can lead to a previously unexplored phenomenon we term ``entrapment'': The random walk may become trapped in a small region of the network, resulting in highly correlated updates and severely degraded convergence. To address this issue, we propose Metropolis–Hastings with L\'{e}vy Jumps (MHLJ), which introduces occasional long-range transitions to restore exploration while respecting local information constraints. We establish a convergence rate that explicitly characterizes the roles of data heterogeneity, network spectral gap, and jump probability, and demonstrate through experiments that MHLJ effectively eliminates entrapment and significantly speeds up decentralized learning.

\end{abstract}

\begin{IEEEkeywords}
Decentralized learning, distributed learning, Markov SGD, weighted sampling
\end{IEEEkeywords}


\section{INTRODUCTION}
\IEEEPARstart{T}{raditional} machine learning approaches typically store data and train models on a single server. This paradigm becomes increasingly problematic in modern networked environments where data are generated and stored across a large number of distributed devices such as mobile phones, sensors, and edge nodes. Moving large volumes of data to a central server introduces significant communication overheads and raises privacy concerns. These challenges have motivated extensive research on distributed learning \cite{zinkevich2010parallelized,richtarik2016parallel}. 

A commonly studied architecture is centralized distributed learning, where local devices communicate with a central parameter server. Although effective, this architecture suffers from communication bottlenecks \cite{karimireddy2020scaffold} and creates a single point of failure if the server becomes unavailable \cite{gupta2021localnewton,guerraoui2018hidden}. These limitations are particularly pronounced in mobile and edge networks where connectivity may be sparse or unreliable.

\begin{figure}[t]
\centering
\includegraphics[width=0.8\linewidth]{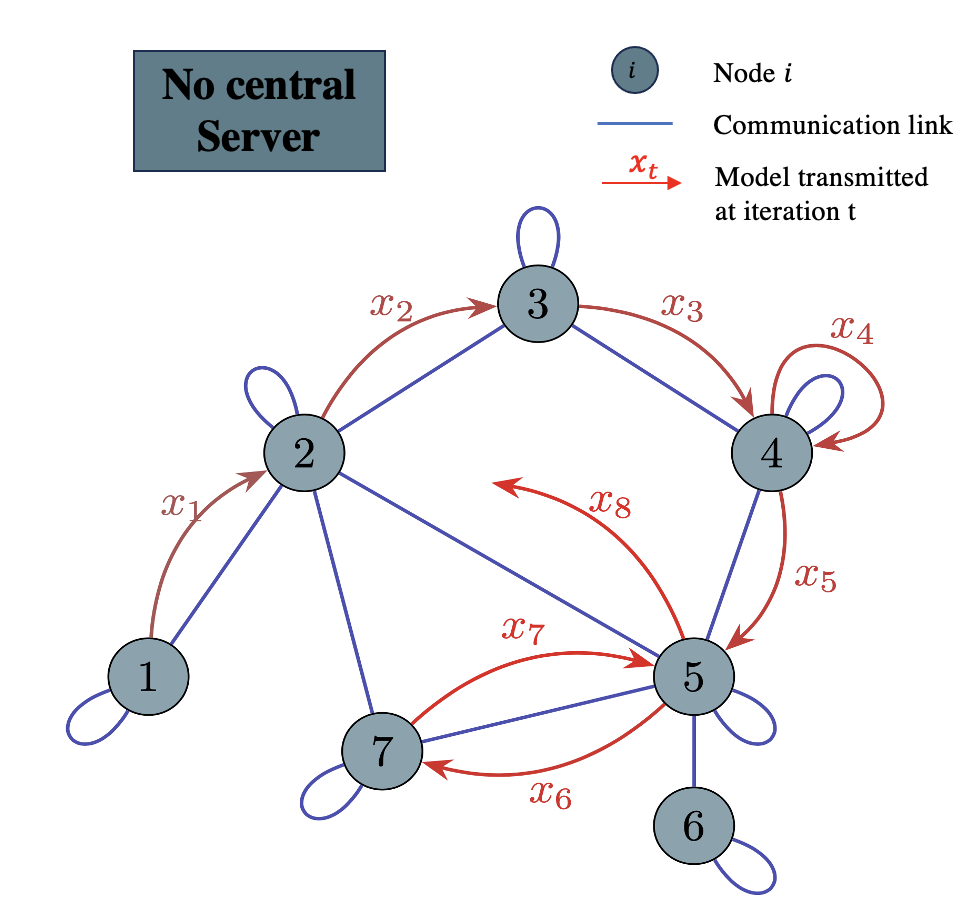}
\caption{\small Decentralized learning via random walk. The model $x$ is carried across the network through a sequence of peer-to-peer transmissions (red arrows). Each visited node performs a local update using its private data before forwarding the model to a neighbor.}
\label{Learning}
\vspace{-0.5cm}
\end{figure}

Decentralized learning eliminates the dependency on a central server by allowing devices to communicate only with their neighbors. In this work, we study decentralized learning via \emph{random walks} (RWs), as illustrated in Figure~\ref{Learning}. The training data reside locally on network devices (nodes), and the learning model is propagated across the network through peer-to-peer communication. At each step, the node currently holding the model performs a local update using its private data, then forwards the updated model to one of its neighboring devices. This approach significantly reduces communication overheads and avoids centralized coordination \cite{kairouz2021advances}.

Formally, the decentralized learning objective is
\begin{align}
    \min_{x\in \mathbb{R}^d}\frac{1}{|V|}\sum_{v\in V}f_v(x),\label{obj}
\end{align}
where $f_v$ denotes the local loss function associated with node $v$, determined by its local dataset.

Existing decentralized learning methods generally fall into two main categories. The first is \emph{consensus-based algorithms, particularly gossip-based methods}, which have been extensively studied (see \cite{boyd2006randomized,dimakis2006geographic,nedic2009distributed,koloskova2020unified,dimakis2010gossip,liu2025decentralized} and references therein). In these approaches, each node maintains a local model, and convergence is achieved via iterative averaging with neighboring nodes. The second category is \emph{random-walk-based algorithms} (e.g., \cite{johansson2010randomized,ayache2021private,hendrikx2023principled,even2023stochastic} and references therein), where a single model propagates through the network as a token. Recently, these methods have attracted attention due to their minimal communication requirements: only one active link per iteration, making them well-suited for bandwidth- or energy-constrained networks.

In this work, we focus on random-walk stochastic gradient descent (SGD) \cite{robbins1951stochastic,bertsekas2011incremental,nemirovski2009robust}. Starting from an initial model $x^0$ and an initial node $v_0$, the node holding the model at iteration $t$ computes a stochastic gradient $\hat g_{v_t}=\nabla f_{v_t}(x^t)$ using its local data,  updates the model, and then forwards the model to a neighboring node governed by the transition matrix $P$.

The design of this transition rule plays a critical role in determining how efficiently the learning process explores the network. Three representative designs have been studied:

\begin{enumerate}

\item \textbf{Uniform neighbor selection.}  
The simplest strategy forwards the model to a uniformly selected neighbor, i.e.,
$P(v,u)=\tfrac{1}{\deg(v)}$, where \(\deg(v)\) denotes the degree of node \(v\), i.e., the number of neighbors of \(v\) in the graph.
However, this leads to a biased visitation pattern where high-degree nodes are visited more frequently.  

\item \textbf{The Metropolis-Hastings transition.}  
A widely used alternative applies the Metropolis-Hastings rule \cite{metropolis1953equation,hastings1970monte} to construct a transition matrix with a uniform stationary distribution \cite{johansson2007simple}:
\[
P(v,u)= \min\{\tfrac{1}{\deg(v)} ,\tfrac{1}{\deg(u)}\}, u\neq v,(u,v)\in E.
\]
This design aims to emulate centralized SGD, where data samples are drawn uniformly.

\item \textbf{Weighted sampling.}  
More generally, one may target a desired stationary distribution $\pi$ that reflects the importance of different nodes or datasets. The corresponding Metropolis-Hastings transition rule  is \cite{ayache2021private, ayache2023walk}
\begin{align*}
P(v,u)=\min\left\{\tfrac{1}{\deg(v)},\tfrac{\pi_u}{\deg(u)\pi_v}\right\}, 
 u\neq v,(u,v)\in E .
\end{align*} 

\end{enumerate}

In this paper, we focus on the third design, where the target distribution $\pi$ corresponds to a weighted sampling strategy. In centralized learning, importance sampling can significantly speed up convergence by prioritizing informative data points \cite{needell2014stochastic}.  Similar speed-up effects have also been observed in decentralized learning settings \cite{ayache2021private}, where weighted sampling is implemented through local interactions.   However, we show that when such weighted sampling is implemented through Metropolis-Hastings transitions over sparse communication networks, the resulting random walk may exhibit an over-weighting effect that we term \emph{entrapment}. 
Specifically, the learning token may remain confined to a small subset of nodes or a localized region of the network for extended periods, severely limiting exploration of the network and slowing the overall learning process. 

To address this issue, we introduce a new transition design that perturbs the \mh dynamics with random L\'{e}vy-type long-range jumps \cite{riascos2012long}. These jumps allow the learning process to escape from locally trapped regions of the network while preserving the decentralized communication structure. We demonstrate that this mechanism significantly improves network exploration and speeds up convergence.

\subsection{Previous Work}
The advent of large-scale data has spurred significant interest in distributed learning \cite{zinkevich2010parallelized,richtarik2016parallel}, with federated learning \cite{kairouz2021advances} emerging as a prominent framework in recent work. However, this approach's reliance on a central server introduces challenges such as communication bottlenecks and privacy concerns \cite{pustozerova2020information}, limiting its applicability in some scenarios. In contrast, decentralized learning eliminates the need for a central server and leverages local communication between edge devices organized as a communication network.

Existing decentralized learning approaches can be categorized into two main types: consensus/gossip-type algorithms \cite{boyd2006randomized,dimakis2006geographic,nedic2009distributed,koloskova2020unified,dimakis2010gossip,liu2025decentralized} and random-walk-type algorithms \cite{ram2009incremental,johansson2010randomized,ayache2021private}. In gossip-type algorithms, each node has its own model. The models are individually updated through communication with their neighbors. 
However, gossip-type algorithms require significant communication overhead \cite{even2023stochastic} and necessitate the design of a gossip matrix \cite{boyd2006randomized}. On the other hand, in random-walk-type algorithms, only one model is passed from one node to another. The model is updated continuously, and the transition decision is determined locally. In this paper, we focus on random-walk-based decentralized learning algorithms.

\textbf{Random-walk learning}. Most of the prior work on decentralized learning has focused on gossip-based or consensus-based methods. Recently, there has been growing interest in random-walk-based learning. The original research by Johansson et al. \cite{johansson2007simple, johansson2010randomized},
  marked the initial exploration of random-walk learning with the (sub-)gradient method and provided a convergence guarantee. Subsequently, Ram et al. \cite{ram2009incremental} extended the results from \cite{johansson2007simple} to accommodate changing topologies and noisy gradient responses. The convergence guarantee for the random-walk first-order algorithm was also established in their study \cite{johansson2010randomized,ram2009incremental}.
Wai et al. \cite{wai2018sucag} delved into utilizing curvature information to speed up the convergence rate of random-walk SGD. Sun et al.   \cite{sun2022adaptive} investigated random-walk SGD with adaptive step sizes and momentum.

Random-walk learning, also sometimes referred to as incremental learning \cite{lopes2007incremental, rabbat2005quantized, zhao2014asynchronous}, is a subclass of learning under Markovian sampling or, more broadly, under ergodic sampling. 
Duchi et al. \cite{duchi2012ergodic} explored ergodic sampling for the mirror descent method. Sun et al.  \cite{sun2018markov} presented non-convex results for SGD under Markovian sampling and showed numerical potential speed-up using a non-reversible Markov chain. Dorfman and Levy \cite{dorfman2022adapting} examined the AdaGrad method with Markovian sampling. Even  \cite{even2023stochastic} investigates variance reduction methods under Markovian sampling.
 The works discussed above primarily focus on gradient-based methods. Mao et al.  \cite{mao2020walkman} introduced an ADMM-type approach, the Walkman algorithm, to random-walk learning, while Hendrikx  \cite{hendrikx2023principled} proposed a general framework for analyzing the random-walk algorithms under the Bregman block coordinate descent perspective.  

 The aforementioned works primarily focus on scenarios where the random walk's stationary distribution is uniform across nodes. In this paper, we build closely on the work of Ayache et al. \cite{ayache2021private, ayache2023walk}, who expanded random-walk learning beyond uniform sampling by introducing techniques such as weighted sampling and multi-armed bandits.

\textbf{Weighted sampling}. The other direction closely related to this work is weighted sampling. The main idea is to assign different probabilities to nodes or data points based on their importance, as captured by the local Lipschitz constant (defined later) or other metrics, thereby improving the efficiency of learning algorithms by prioritizing more informative samples. Another way to understand weighted sampling is that it averages the local Lipschitz constants, enabling potentially larger step sizes in updates.
Strohmer and Vershynin \cite{strohmer2009randomized} first showed a variant
of the Kaczmarz method using a biased selection method, which selects rows with probability
proportional to their squared norm, and converges linearly.
Needell et al. \cite{needell2014stochastic} generalized the result to general SGD and showed that sampling the data with probability proportional to the Lipschitz constant of the local loss function can speed up convergence when the data are heterogeneous. Zhao and Zhang \cite{zhao2015stochastic} showed that using a distribution proportional to the gradient norm can minimize the variance of the stochastic gradient and, therefore, speed up the convergence. Csiba et al. \cite{csiba2018importance} first studied weighted minibatch sampling in 2018. All these works focused on centralized scenarios.

\subsection{Contributions}
 
Our main contribution is introducing random multi-hop jumps to mitigate entrapment, preventing random walks from becoming trapped in local regions, and allowing them to mix faster. This improves exploration and learning efficiency in resource-constrained, sparse networks while respecting locality constraints. We develop a unified algorithmic and theoretical framework to analyze the newly proposed method, substantiated with extensive numerical experiments. Our main contributions can be summarized as follows:

\begin{itemize}

\item \textbf{Random Walk Entrapment.}
We identify a previously unexplored challenge in random-walk-based learning. When weighted sampling is implemented via the Metropolis–Hastings algorithm, the resulting random walk can become trapped within a region of the network. This effect can arise when there is a significant disparity in node importance and is further exacerbated by limited network connectivity (such as in a ring topology). As a consequence, model updates become highly correlated, gradient aggregation becomes more biased, and convergence is substantially degraded. This phenomenon reveals a fundamental limitation of decentralized weighted sampling that does not arise in centralized or fully connected settings.

\item \textbf{Metropolis-Hastings with L\'{e}vy Jumps (MHLJ).}
To avoid entrapment while respecting local communication constraints, we propose a new algorithmic framework termed Metropolis-Hastings with L\'{e}vy Jumps (MHLJ). The method introduces randomly long-range jump transitions that intentionally break the detailed balance condition to enhance exploration and reduce the risk of entrapment. The jumps are executed as multiple consecutive one-hop simple random walk moves (uniformly choosing the next node over all neighbors) without update, which can be decided in a decentralized way using only local information.
The MHLJ provides a tunable trade-off between exploration efficiency, captured by the spectral gap, $\eta$, and perturbation bias, captured by the probability to jump, $p_J$, enabling system designers to balance convergence speed against sampling bias.  

\item \textbf{Convergence rate analysis under network constraints.}

 We develop a perturbation-based framework to analyze the convergence of the proposed MHLJ algorithm, building on the auxiliary sequence technique introduced in  \cite{even2023stochastic}. In this formulation, the transition probabilities are decomposed into a baseline-weighted-sampling transition and an additional perturbation induced by the Lévy jumps.
The main result is presented in Theorem~\ref{thmhlj}, and characterizes: 1. the convergence rate of the MHLJ algorithm, which is comparable to the centralized weighted sampling except for the mixing factor characterized by the spectral gap of the transition matrix; 2. the error gap due to the deviation from the stationary distribution of the weighted sampling by adding jumps, which depends on the jump probability. The error gap is further validated through simulations.

\item \textbf{Simulation on sparse network models.} 
  Through experiments on synthetic heterogeneous datasets and multiple sparse graph families, we demonstrate that MHLJ consistently eliminates entrapment and speeds up convergence compared with uniform sampling and weighted sampling. We further provide two practical mechanisms to eliminate the residual error gap in deployment: (i) gradually decreasing the jump probability, and (ii) switching to uniform sampling after sufficient mixing. These strategies offer implementation guidance for real-world decentralized learning systems.

\end{itemize}

\subsection{Organization}
 The rest of the paper is organized as follows: Section II introduces the problem setting. Section III characterizes the entrapment problem in random-walk learning. 
Section IV introduces our proposed algorithm, MHLJ, together with the main theoretical results.
 We present the details and discussion of the simulation in Section V. Finally, we prove the theoretical convergence result for MHLJ in Section VI.

   \section{Problem Setting}
   In this section, we introduce the problem setting and necessary background. See Table~\ref{table} for the notations.

    \subsection{Communication Network and Objective Function}
    We consider a communication network represented by a graph $G=(V, E)$, where $V$ is the set of nodes, and $E\subseteq V\times V$ represents the communication links between nodes. Nodes that are connected can communicate with each other. We assume that each node in the graph has a self-loop. Each node $v$ of the network has its local data $(A_v, y_v)$,  which induces a local loss function $f_v(x)=f(x;A_v,y_v)$, where $f$ is a strongly convex and Lipschitz smooth loss function and $x\in \mathbb{R}^d$ is the model parameter. The goal is to find a decentralized algorithm to minimize the average of local functions 
    using only local communications without the help of a central server.
 The objective function to minimize can be expressed as follows:
\begin{align}
    f(x)=\frac{1}{|V|}\sum_{v\in V}f_v(x),\ x\in \mathbb{R}^d.\label{target}
\end{align}
\subsection{Random-Walk Learning}
We want to design a decentralized algorithm that solves (1) via a random walk.  
A random-walk algorithm  for decentralized optimization \cite{johansson2007simple, ayache2021private} 
consists of the following steps:
\begin{enumerate}
    \item Start from a randomly selected node $v_0$, with the currently visited initial model $x^0$;
    \item At iteration $t$,  $v_t$ updates the model using the stochastic gradient $\hat{g}_{v_t}$ calculated based on the local data:
    \begin{align}
    x^{t+1}=x^t-\gamma_t\hat{g}_{v_t}(x^t), \label{sgd}
\end{align}  
\item Node $v_t$ randomly chooses one of its neighbors (including itself) as $v_{t+1}$, according to a distribution $P(v_t,\cdot)$.
\item  Node $v_t$ passes the model $x^{t+1}$ to node $v_{t+1}$.
\end{enumerate}
The algorithm iterates steps 2, 3, and 4 for a given number of iterations. The model that passes between nodes and their neighbors can be seen as a random walk on the graph $G$. The probability $P(v,\cdot), v\in V$ is the transition probability matrix of this time-homogeneous random walk.

The main question we focus on in this work is how to design the transition matrix $P$ to speed up the convergence rate of the random-walk learning algorithm. 

\subsection{Smooth Functions \& Data Heterogeneity}

  We assume that the local  loss functions are   Lipschitz smooth: 
\begin{definition}
    A function $f(x)$ is L-smooth if 
    \begin{align*}
       \lVert\nabla f(x)-\nabla f(y)\rVert \leq L\lVert x-y\rVert, \textit{ for all }x,y\in \textit{dom}(f),
    \end{align*}
      where $L$ is the gradient Lipschitz constant of the function.
    \end{definition}
In general, the gradient Lipschitz constant depends on the local data and varies from node to node. For instance, in linear regression $f_v(x)=\frac{1}{2}\lVert y_v-x^TA_v\rVert^2$ and $L_v=\lVert A_v\rVert^2$, and in Logistic regression $f_v(x)=y_v x^T A_v-\log(1+e^{x^T A_v})$ and $L_v=\frac{1}{4}\lVert A_v\rVert^2$. In both scenarios, $(A_v,y_v)$ is the local data stored at node $v$. 
     
     We are interested in the scenario where the data owned by the nodes is heterogeneous  \cite{stacke2020measuring, gao2022survey, rajendran2023data}, i.e., not sampled from identical distributions. We will look at the gradient Lipschitz constants $L_v$ of the local loss functions $f_v$ as a proxy for heterogeneity.    We denote 
     $L_{\max}=\max\left\{L_v|v\in V\right\}$, $L_{\min}=min\left\{L_v|v\in V\right\}$, and $\Bar{L}=\frac{1}{|V|}\sum_{v\in V}L_v$.      In particular, we consider the following heterogeneous case:
    \begin{align}
        L_{\min}\approx\Bar{L}\leq L_{\max}.\label{hetedata}
    \end{align}
 
 \subsection{Weighted Sampling}

 In the vanilla centralized SGD scenario \cite{bertsekas2011incremental}, data are stored in a single location and sampled independently and identically from the uniform distribution at each iteration. 
Weighted sampling goes beyond the uniform distribution and samples the data based on a measure of importance, while still maintaining independence and identically distributed (i.i.d.) sampling. 
Of particular importance to our work here is the work of Needell et al. \cite{needell2014stochastic} on centralized SGD, which showed that weighted sampling in a centralized setting can speed up the SGD convergence. It was proposed to use the gradient Lipschitz constant of the local loss function $L_i$ as the importance of data $x_i$, and to sample the data  proportional to its importance, i.e., according to the following distribution: 
\begin{align}
    \pi_{IS}(i)\coloneqq\frac{L_i}{\sum_{i=1}^{N}L_i}, \label{impdist}
\end{align}
where $\pi_{IS}$ is the weighted sampling distribution. To make the stochastic gradient $\hat{g}$ unbiased under this biased sampling, one needs to re-weight the update:
\begin{align}
    x^{t+1} = x^t - \gamma \frac{\Bar{L}}{L_{i_t}}\nabla f_{i_t}(x^t).\label{wsgd}
\end{align}
\subsection{Assumptions}
We need the following general assumptions on the local loss functions to get our convergence result.

\textit{Assumption 1.}
    Local Lipschitz smoothness:
    \begin{align*}
         \norm{\nabla f_v(x)-\nabla f_v(y)} \leq L_v\norm{x-y},\ \forall x,y\in\mathcal{X}, \forall v\in V.
    \end{align*}
    
\textit{Assumption 2.}
    Local strong convexity:
    \begin{align*}
        f_v(y)-f_v(x)\geq \ip{\nabla f_v(x),y-x}+\frac{\mu}{2} \norm{y-x}^2,\ \forall v\in V.
    \end{align*}
    
\textit{Assumption 3.}
    Bounded norm of the local gradient at the global optimum:
    \begin{align*}
        \norm{\nabla f_v(x^*)}^2\leq \sigma_{\max}^2,\ \forall v\in V,
    \end{align*}
    where $x^*$ is the minimizer of (\ref{target}). We also denote $\sigma^2_*=\frac{1}{|V|}\sum_{v\in V}\norm{\nabla f_v(x^*)}^2$.

A random walk on a graph can be seen as a Markov chain. We need the following concepts of Markov chains. Let $P$ be the transition matrix of this chain. We consider only the case where $P$ is irreducible and aperiodic, then regardless of the initial distribution $\mu$, we have 
\begin{align}
    \lim_{t\to\infty}\mu P^t=\pi \label{MClim}
\end{align}
where $\pi$ is the stationary distribution of $P$ that satisfies $\pi=\pi P$, see \cite[Theorem 1.19]{durrett1999essentials}. We make an assumption about the random walk. 

\textit{Assumption 4.} $P$ is irreducible and aperiodic, and $v_0$ is sampled from the stationary distribution of $P$.

\section{The Entrapment Problem}
Before summarizing our main results, we identify and explain the entrapment problem that we tackle in this work. The driving motivation is to implement weighted sampling in a decentralized way to achieve convergence speedups similar to those observed in the centralized case. 
The authors of \cite{ayache2021private} proposed implementing weighted sampling in random walk learning by designing the random walk such that its stationary distribution is the weighted sampling distribution defined in \eqref{impdist}. This is achieved  by designing the transition probability matrix using the \mh algorithm:
 \begin{align}
     P_{IS}(v,u)=\left\{\begin{array}{cc}
          \frac{1}{\deg(v)}\min\{1, \frac{\deg(v) L_u}{\deg(u) L_v}\},& v\neq u,\\
          1-\sum_{w:(v,w)\in E}P_{IS}(v,w),&v= u.
     \end{array}
     \right. \label{impmh}
 \end{align}

       Under the transition matrix \eqref{impmh}, the random walk asymptotically visits each node with a frequency given by \eqref{impdist}, mimicking the weighted sampling in the centralized scenario. The \mh algorithm requires only the local information, i.e., the information from the neighbors of a node, to achieve a desired stationary distribution. Figure~\ref{Erimp} shows the speedup using weighted sampling in random-walk learning on the \er (1000, 0.1) network.  
        
However, we show that when the data are heterogeneous and the graph is not ``well-connected", the random walk designed by \eqref{impmh} might be entrapped in those important nodes and cause a slowdown in convergence. Consider, for example, a ring graph as in Figure~\ref{fivring}: one node holds extreme data, leading to a large gradient Lipschitz constant compared with the other nodes on the ring. In this toy example, the transition matrix generated by the Metropolis-Hastings algorithm is shown in Figure~\ref{ringMC}; the transition probabilities indicate that the random walk tends to stay in the ``important'' node.
        As a result, for most sample paths, the random walk becomes easily entrapped in those important nodes for long periods. From an optimization perspective, the entrapment problem forces the algorithm to repeatedly update the model using the same data, thereby pushing it toward a local minimum of the loss function. In other words, entrapment marginalizes the effects of these updates globally, thereby slowing the convergence of learning. We refer to this as the entrapment problem. 
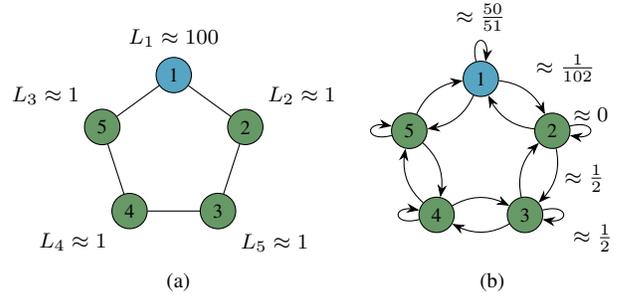
\begin{figure}[t]
\centering
\subfloat[]{\begin{tikzpicture}[>=Stealth, node distance=2cm, scale=0.5]
    \tikzstyle{node} = [circle, draw, minimum size=1.5em]

    \node[node, fill=green!20!gray,scale=0.75] (A) at (18:2cm) {2};
    \node[node, fill=cyan!50!gray,scale=0.75] (B) at (90:2cm) {1};
    \node[node, fill=green!20!gray,scale=0.75] (C) at (162:2cm) {5};
    \node[node, fill=green!20!gray,scale=0.75] (D) at (234:2cm) {4};
    \node[node, fill=green!20!gray,scale=0.75] (E) at (306:2cm) {3};

    \node [above right=0.01cm and 0.01cm of A] {\footnotesize $L_2\approx 1$};
    \node [above=0.01cm of B] {\footnotesize $L_1\approx 100$};
    \node [above left=0.01cm and 0.01cm of C] {\footnotesize$L_3\approx 1$};
    \node [below left=0.001cm and 0.01cm of D]  {\footnotesize$L_4\approx 1$};
    \node [below right=0.01cm and 0.01cm of E]  {\footnotesize$L_5\approx 1$};

    \draw[-] (A) to  (B);
    \draw[-] (B) to  (C);
    \draw[-] (C) to  (D);
    \draw[-] (D) to  (E);
    \draw[-] (E) to (A);

\end{tikzpicture}\label{fivring}}
\subfloat[ ]{
    \begin{tikzpicture}[>=Stealth, node distance=2cm, scale=0.5]
    \tikzstyle{node} = [circle, draw, minimum size=1.5em]

    \node[node, fill=green!20!gray,scale=0.75] (A) at (18:2cm) {2};
    \node[node, fill=cyan!50!gray,scale=0.75] (B) at (90:2cm) {1};
    \node[node, fill=green!20!gray,scale=0.75] (C) at (162:2cm) {5};
    \node[node, fill=green!20!gray,scale=0.75] (D) at (234:2cm) {4};
    \node[node, fill=green!20!gray,scale=0.75] (E) at (306:2cm) {3};

    \node at (22:2.5cm)  {};
    \node at (90:3cm)  {};
    \node at (157:2.5cm)  {};
    \node at (234:2.5cm)  {};
    \node at (306:2.5cm)  {};

    \draw[->] (A) to[bend left]  (B);
    \draw[->] (B) to[bend left] (C);
    \draw[->] (C) to[bend left] (D);
    \draw[->] (D) to[bend left] (E);
    \draw[->] (E) to[bend left]  (A);

    \draw[->] (B) to[bend left] node[above right] {\footnotesize$\approx\frac{1}{102}$} (A);
    \draw[->] (C) to[bend left] (B);
    \draw[->] (D) to[bend left] (C);
    \draw[->] (E) to[bend left] (D);
    \draw[->] (A) to[bend left] node[right] {\footnotesize$\approx\frac{1}{2}$}(E);

    \draw[->] (A) to[loop right]node[above] {\footnotesize$\approx0$}(A);
    \draw[->] (B) to[loop above] node[above] {\footnotesize$\approx\frac{50}{51}$}(B);
    \draw[->] (C) to[loop left] (C);
    \draw[->] (D) to[loop left] (D);
    \draw[->] (E) to[loop right] node[below right]{\footnotesize$\approx \frac{1}{2}$} (E);
\end{tikzpicture}%
 \label{ringMC}}
 
\caption{(a) An example of ring topology with five nodes. The stored data are heterogeneous because one of the local datasets yields a large local Lipschitz constant. (b) In the Markov chain representation of the random walk on the graph in (a), the probability of moving out of the `important' node is tiny compared with the probability of staying locally. Some similar transition probabilities are omitted. } 
\label{ring and chain}
\end{figure}

In the following part of this paper, we will mainly focus on one representative example: the heterogeneous data\footnote{See the Numerical Result Section for more details about the data.}, as in the previous discussion, distributed over a ring network with 1000 nodes. Figure~\ref{simentrap}  shows the simulation of this representative example; the convergence rate of \mh \  weighted sampling is dramatically slower, contrary to previous work's observations. Note that the entrapment problem occurs not only on the ring network, but also on other ``sparse'' networks; see section V for more discussion.

To understand the cause of the entrapment problem, notice that the transition probability of the random walk is designed by the Metropolis-Hastings algorithm; thus, the transition matrix $P_{IS}$ satisfies the detailed balance condition
\begin{align}
     \pi_{IS}(v)P_{IS}(v,u)=\pi_{IS}(u)P_{IS}(u,v).
\end{align}
When some nodes have much larger local Lipschitz constants than others, and the graph is sparse, jumping out of an `important' node has low probability. Figure~\ref{ringMC} gives an example of the Markov chain. Once the walk visits an important node, the probability of staying locally is substantial. See Figure~\ref{ring and chain} for an illustration.

\begin{table}[t]
\centering
\resizebox{0.5\textwidth}{!}{\begin{tabular}{|c|c|}
\hline
  & The Metropolis-Hastings \\$P_{IS}$&transition matrix with weighted\\
&stationary distribution \eqref{impdist}
\\ \hline
 $P_{\textit{L\'{e}vy}}$      & The transition probability matrix of L\'{e}vy jump
\\ \hline
$x^0$      &   Initial model
\\ \hline
 $v_0$    & Initial node

\\ \hline
 $\gamma$      & Step size
 \\ \hline
 $T$      & Number of iterations
\\ \hline
 $p_{J}$      & Probability of making jump
\\ \hline
 $p_d$      & Parameter of choosing jump distance
\\ \hline
 $L_v$      & Local Gradient Lipschitz constant at node $v$
\\ \hline
 $L_{\max}$      &  $\max_v L_v$
\\ \hline
 $L_{\min}$      & $\min_v L_v$
\\ \hline
 $\Bar{L}$      & $\sum_{v\in V}L_v/|V|$
\\ \hline
 $\mu$      & Strongly convex constant
\\ \hline
 $\eta$      & The spectral gap of transition probability $P$
\\ \hline
 $\sigma_{\max}^2$      & Largest $L_2$ norm of local gradient at optimum
\\ \hline
$\mathsf{Geom}(p,r)$ & Geometric distribution with probability $p$ truncated at $r$
\\ \hline

\end{tabular}}
\caption{Important notations used in the paper.}\label{table}
\end{table}

\section{Main Results} 
\subsection{Overcoming Entrapment}
Our first contribution is the proposal of a new algorithm, Metropolis–Hastings with L\'{e}vy Jumps (MHLJ), to mitigate the entrapment problem. The details of MHLJ are described in Algorithm~\ref{alg:jump}.  
The main idea is to modify the random walk’s transition by perturbing the standard Metropolis-Hastings transition probability with L\'{e}vy jumps. This perturbation breaks the detailed balance condition and helps mitigate the previously identified issue of local entrapment.

 Our second contribution is the establishment of convergence guarantees for MHLJ, which characterizes the dependence of the convergence rate on the objective function parameters ($\mu, L, \textit{and } \sigma^2$) and the Markov chain ($\eta$), as well as the perturbation-induced error gap. To characterize the error gap, it can be shown that the transition probability can be written as the \mh transition probability $P_{IS}$ perturbed by a L\'{e}vy jump transition matrix $P_{\textit{L\'{e}vy}}$, whose closed form is 
\begin{align*}
    P_{\textit{L\'{e}vy}}= \sum_{i=1}^{r}\frac{p_d(1-p_d)^{i-1}}{1-(1-p_d)^{r}}\operatorname{diag}\{A_G^i \textbf{1}\}^{-1}A_G^i,
\end{align*}  
  where $A_G^i$ is the $i$th power of the adjacency matrix $A_G$ of graph $G$. Then, we have 
  \begin{align*}
      P_{MHLJ} = P_{IS}+p_J(P_{\textit{L\'{e}vy}}-P_{IS}).
  \end{align*}
\begin{thm}[Convergence of MHLJ]
    Under Assumptions 1-4, for a given step $T\geq 1$, with some  step size $\gamma<\frac{1}{\Bar{L}}$, the output $x^T$ of MHLJ after $T$ iterations satisfies:
\begin{equation}\label{eq:thmhlj}
\begin{aligned}
\E{\norm{x^T-x^*}^2} \leq\;&
2(1-\gamma\mu)^T\norm{x^0-x^*}^2 \\
+& \tfrac{\gamma\Bar{L}\sigma_{\max}^2}{\eta\mu^2} + \tfrac{\Bar{L}^3 p_J^2  \norm{\pi-\Tilde{\pi}}_{TV}^2
   \sigma_{\max}^2 }{\mu^3 L_{\min}^2}.
\end{aligned}
\end{equation}
            where $\Tilde{\pi}$ and $\eta$ are the stationary distribution and spectral gap of $P_{MHLJ}$, respectively, and $\pi$ is the stationary distribution of $P_{IS}$.\label{thmhlj}
            \end{thm}
            \begin{algorithm}[t]
\caption{Weighted Sampling using the Metropolis-Hastings with L\'{e}vy Jumps (\algname)}\label{alg:jump}
\KwInput {$G=(V,E)$, $L_v,v\in V$, $P_{IS}$, $x^0$, $v_0$, $\gamma$, $T$, $p_{J}$, $p_d$, $r$.}
\KwOutput{$x^{T}$.}

\For{t = 0,1,...,T-1}
{\tcc{Model updating}
$x^{t+1}=x^t-\gamma \frac{\Bar{L}}{L_{v_t}}\nabla f_{v_t}(x^t)$

$J\sim\mathsf{Bern}(p_{J})$

\If{$J = 0$}
{ \tcc{the Metropolis-Hastings transition}
$v_{t+1}$ $\sim$ $P_{IS} (v_{t},\cdot)$

}

\Else 
{\tcc{L\'{e}vy Jumps}
$d\sim \mathsf{Geom}(p_d,r)$\\
\While{$d\neq 0$}
{
$v'\sim \mathsf{Unif}(\mathcal{N}_{v_t})$\\
$v_{t} = v'$ \\
$d=d-1$
}
$v_{t+1}=v_t$
}
}
\Return{$x_{T}$}
\end{algorithm}

 The right-hand side of \eqref{eq:thmhlj} contains three terms. The first term converges linearly. The second term captures the variance induced by stochastic gradients under a constant step size. This term dominates and makes convergence sublinear, as we see later, with \(\eta\) denoting the spectral gap of the random walk, reflecting the sampling dependency imposed by the graph topology. The third term captures the perturbation error, which is independent of the step size. We can control this error by choosing a small \( p_{J} \); however, if \( p_{J} \) is too small, the random walk may struggle to escape entrapment, reducing the spectral gap. By balancing these factors with an appropriate step size, we obtain the following convergence result for the algorithm.  
\begin{figure}[!t]
\centering
\subfloat[]{%
  \includegraphics[width=0.95\linewidth]{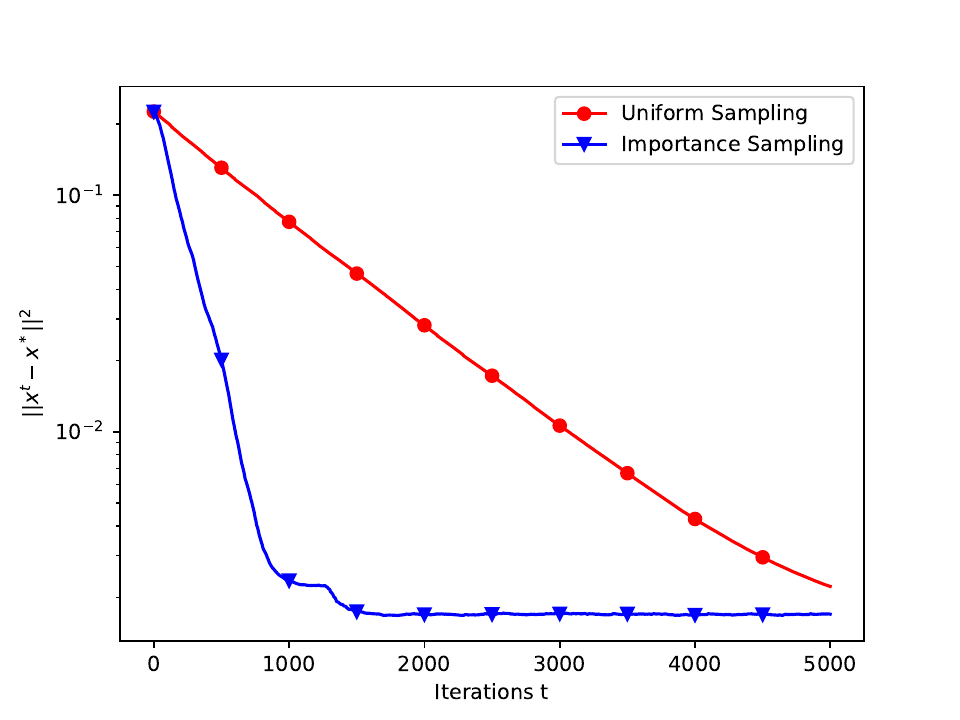}%
  \label{Erimp}%
}
\vspace{0.2em}
\subfloat[]{%
  \includegraphics[width=0.95\linewidth]{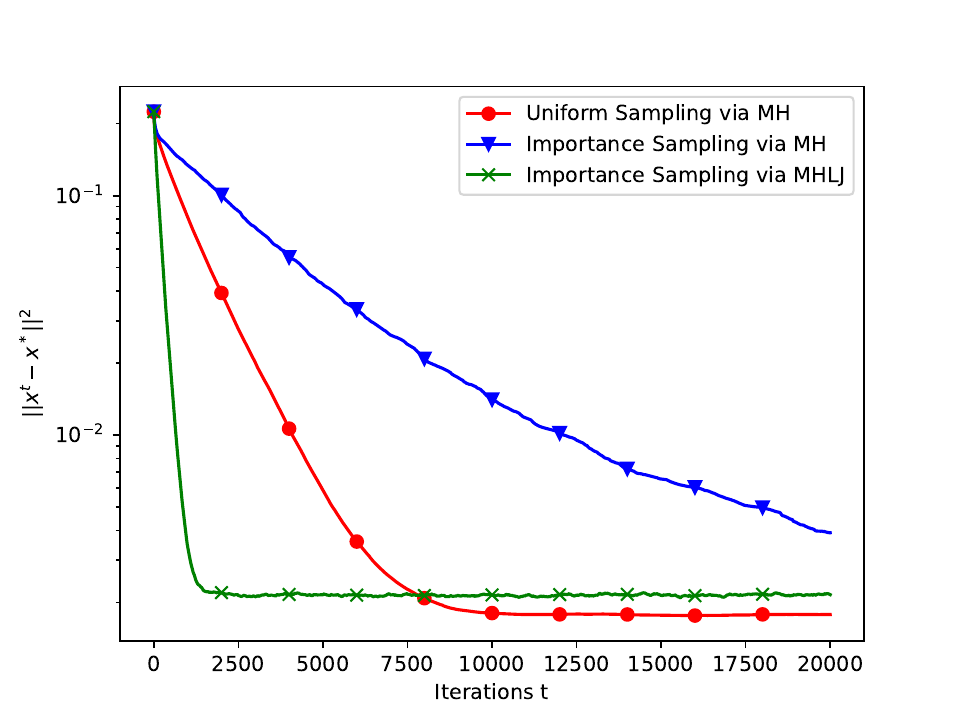}%
  \label{simentrap}%
}
\caption{(a) Regression model trained on a synthetic heterogeneous data set over an \er (1000, 0.1) network with 1000 nodes. We compare uniform sampling with the Metropolis–Hastings transition probability, and weighted sampling with the same transition probability. (b) A regression model was trained on a synthetic, heterogeneous dataset over a Ring(1000) network with 1000 nodes. We compare uniform sampling with the Metropolis–Hastings transition probability, weighted sampling with the Metropolis–Hastings transition probability, and weighted sampling with \algname.}
\label{fig:twoparts}
\end{figure}
 
\begin{coro}
    Under the same setting as in Theorem~\ref{thmhlj}, choose $\gamma\leq\min\{\frac{1}{\Bar{L}},\frac{\epsilon L_{\min}\eta \cdot \mu^2}{2\Bar{L}^2\sigma_*^2}\}$, we have after 
    \begin{align*}
        T =  \left\lceil \log{\tfrac{4\norm{x^0-x^*}^2}{\epsilon}}\left(\tfrac{\Bar{L}^2\sigma_*^2}{\epsilon \eta L_{\min }\mu^3}\right)\right\rceil 
    \end{align*}
    iterations,
 
    \begin{align}
        \E{\norm{x^T-x^*}^2}\leq \epsilon+ \tfrac{\Bar{L}^3 p_{J}^2\norm{\pi-\Tilde{\pi}}_{TV}^2\sigma_{\max}^2}{L_{\min}^2\mu^4} .
    \end{align}\label{CompMHLJ}
\end{coro}

 \subsection{Revisiting Random-Walk SGD}
 Our third contribution is to recover the fast convergence rate of weighted sampling in random-walk learning. Different from \cite{ayache2021private}, where the convergence result was given for convex and Lipschitz smooth objective functions, we give the fast convergence rate for the strongly convex and Lipschitz smooth objective functions. This can be obtained as a special case of Theorem~\ref{thmhlj} without jump. This result is comparable to the convergence of weighted sampling SGD in the centralized case given in \cite{needell2014stochastic}.
\begin{thm}
    Under Assumptions 1-4, for a given step $T\geq 1$, with some  step size $\gamma<\frac{1}{\Bar{L}}$, 
     the output of random-walk SGD using weighted sampling via the Metropolis-Hastings algorithms after $T$ iterations satisfies: \begin{align}
            \E{\norm{x^T-x^*}^2}\leq 2(1-\gamma\mu)^T\norm{x^0-x^*}^2+\tfrac{\gamma\Bar{L}^2\sigma_{*}^2}{L_{\min}\eta_{IS}\mu^2}. \label{rateis}\end{align}
            where $\eta_{IS}$ is the spectral gap of $P_{IS}$ defined in \eqref{impmh}.
             \label{thmnojump}
\end{thm}
 
Theorem~\ref{thmnojump} is quite similar to the result in the centralized scenario. 
The difference from centralized learning is that the noisy term also depends on the spectral gap, i.e., $\eta$. This term reflects the effect of dependent sampling in random-walk learning.
In vanilla SGD, one constraint on the step size that guarantees the convergence is $\gamma\leq\frac{1}{L_{\max}}$. We show that when using weighted sampling, a step size satisfying $\gamma\leq \frac{1}{\Bar{L}}$ is enough. Using a larger step size will cause the algorithm to converge faster initially. However, a larger step size also leads to a larger error term. In random-walk learning, the error term is even more harmful because the spectral gap might be dramatically small when $L_{\max}$ is large. To simplify the dependency, we give the complexity bound below.
\begin{coro}
    Under the setting as in Theorem~\ref{thmnojump}, choose $\gamma\leq\min\{\frac{1}{\Bar{L}},\frac{\epsilon L_{\min}\eta_{IS}\mu^2}{2\Bar{L}^2\sigma_*^2}\}$, we have after 
    \begin{align*}
        T = \left\lceil \log{\tfrac{4\norm{x^0-x^*}^2}{\epsilon}}\left(\tfrac{2\Bar{L}^2\sigma_*^2}{\epsilon \eta_{IS} L_{\min }\mu^3}\right)\right\rceil
    \end{align*}
    SGD iterations, $\E{\norm{x^T-x^*}^2}\leq \epsilon$.\label{coro:ratenojump}
\end{coro}
To better understand Corollary~\ref{coro:ratenojump}, compare it to the centralized case \cite{needell2014stochastic} given by:
\begin{align}
   T= \left\lceil 2 \log  \tfrac{2\norm{x^0-x^*}^2}{\epsilon}  \left(\tfrac{\bar{L}}{\mu}+\tfrac{\bar{L}}{L_{\min}} \cdot \tfrac{\sigma^2}{\mu^2 \varepsilon}\right)\right\rceil.\label{needellcomplex}
\end{align}
We see that Corollary~\ref{coro:ratenojump} gives a similar convergence rate as in \cite{needell2014stochastic}. However, we have two main differences: the constant in our bound is multiplied together, and there is an extra constant $\eta_{IS}$; both terms are due to the sampling dependency in random-walk learning.

\subsection{ Further Discussion}
The L\'{e}vy jumps have been well studied in the network science \cite{kleinberg2000navigation}, and a similar idea has also been implemented in gossip algorithms \cite{dimakis2006geographic}.
A vital property of the added jump is that no global information on the graph is required. Each step of the jump requires only local structure information, i.e., the neighbors of the current node. The details are described in Algorithm~\ref{alg:jump}. See Table~\ref{table} for the meaning of notations.

The novelty of MHLJ is that, after each update, the random walk decides how to move to the next step: (1) it either moves to one of its neighbors according to the Metropolis–Hastings transition probabilities with a large probability; (2) or it makes a L\'{e}vy jump with a small probability. The reason for making jumps with small probability is that we don't want a large deviation from the desired stationary distribution, i.e., the weighted distribution defined in \eqref{impdist}.

\textit{The L\'{e}vy jumps.} When jumping: (1) The current node first decides the jump distance. The jumping distance $d$ is sampled from a truncated Geometric distribution. (2) When the distance $d$ is determined, the model will pass by the nodes $d$ times without updating, and each node passes the model to a uniformly chosen neighbor at each step.  The one-by-one search scheme doesn't violate the local information constraint: the random walk randomly chooses one neighbor at each step without requiring knowledge of the graph's further structure. The uniform sampling strategy during the jump aims to break the detailed balance condition and can push the random walk out of the entrapping area. As a pay-off, the sampling distribution of nodes deviates from the weighted distribution \eqref{impdist}, leading to an error gap, as we will see from our convergence result below.

\begin{remark}
    In previous work \cite{johansson2010randomized, sun2018markov, ayache2021private}, it was shown that designing the transition matrix so that the ergodic expectation of the stochastic gradient equals the true gradient is sufficient for the algorithm to converge to the optimum. However, our result shows that it is also necessary in the non-interpolation scheme. That is, when the ergodic expectation of the stochastic gradient differs from the true gradient, the algorithm will converge to the wrong point, and we can quantify the distance to the true optimum.
\end{remark}
\begin{remark}[Computation vs. communication overheads of \algname] Each iteration in \algname\ ($x$-axis in Figure~\ref{simentrap}) corresponds to one gradient descent update according to (\ref{wsgd}). Figure~\ref{simentrap} shows that \algname\ saves on computation cost since it requires fewer updates to achieve a given accuracy. However, by adding jumps, we allow transitions without updates, thereby increasing communication overheads. For each update, the expected number of transitions (node visits) required can be bounded by
\begin{align*}
    (1-p_J)\cdot 1+p_J\E{d}\leq 1+p_J(\frac{1}{p_d}-1).
\end{align*}
In our example, this upper bound is $1.1$, i.e., at most a $10\%$ increase in the average communication cost.
\end{remark} 

\section{Numerical Results}
\begin{figure*}[ht]
\centering
\subfloat[]{\includegraphics[width=0.43\textwidth]{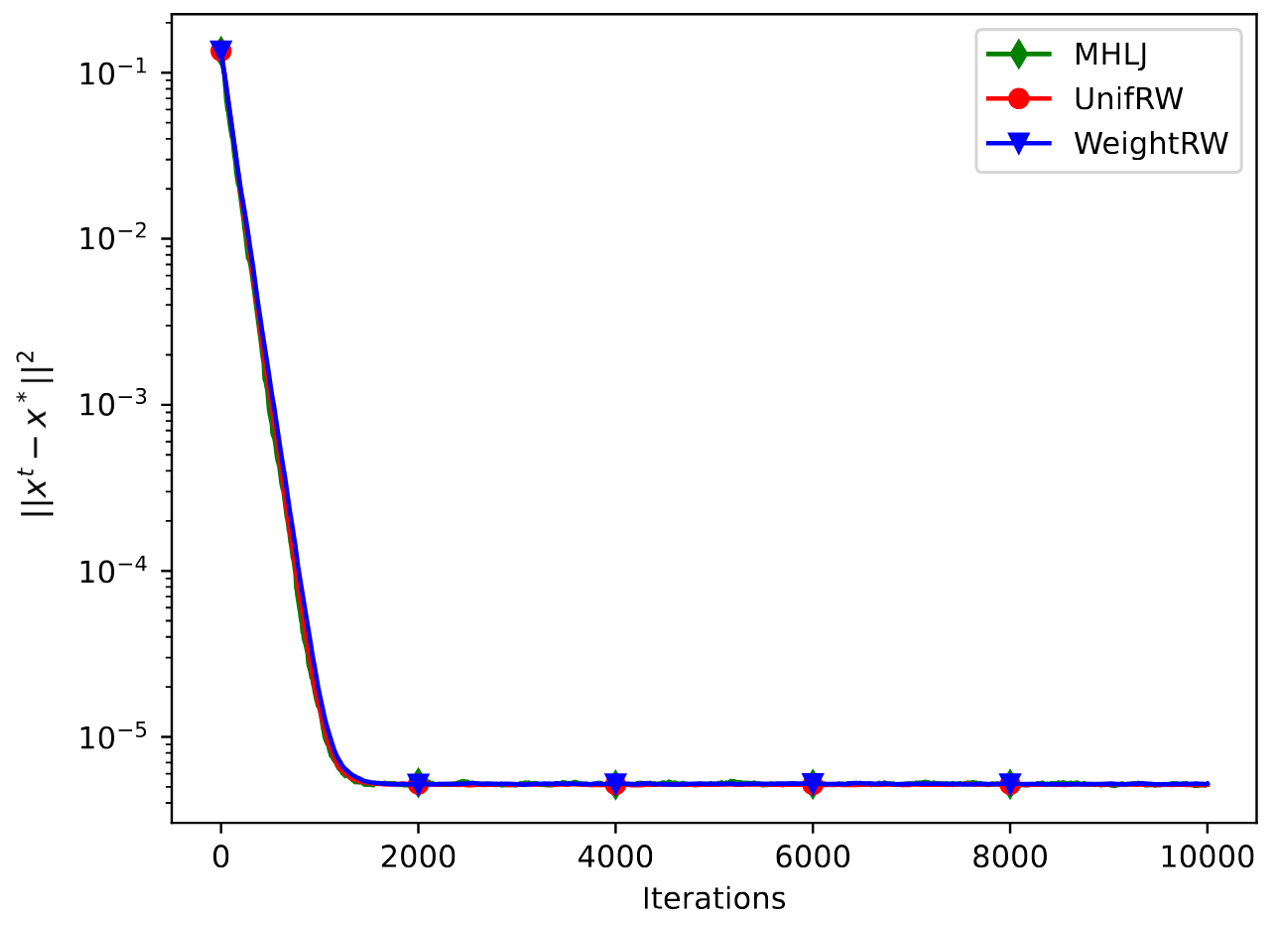}%
\label{homo}}
\hfil
\subfloat[]{\includegraphics[width=0.43\textwidth]{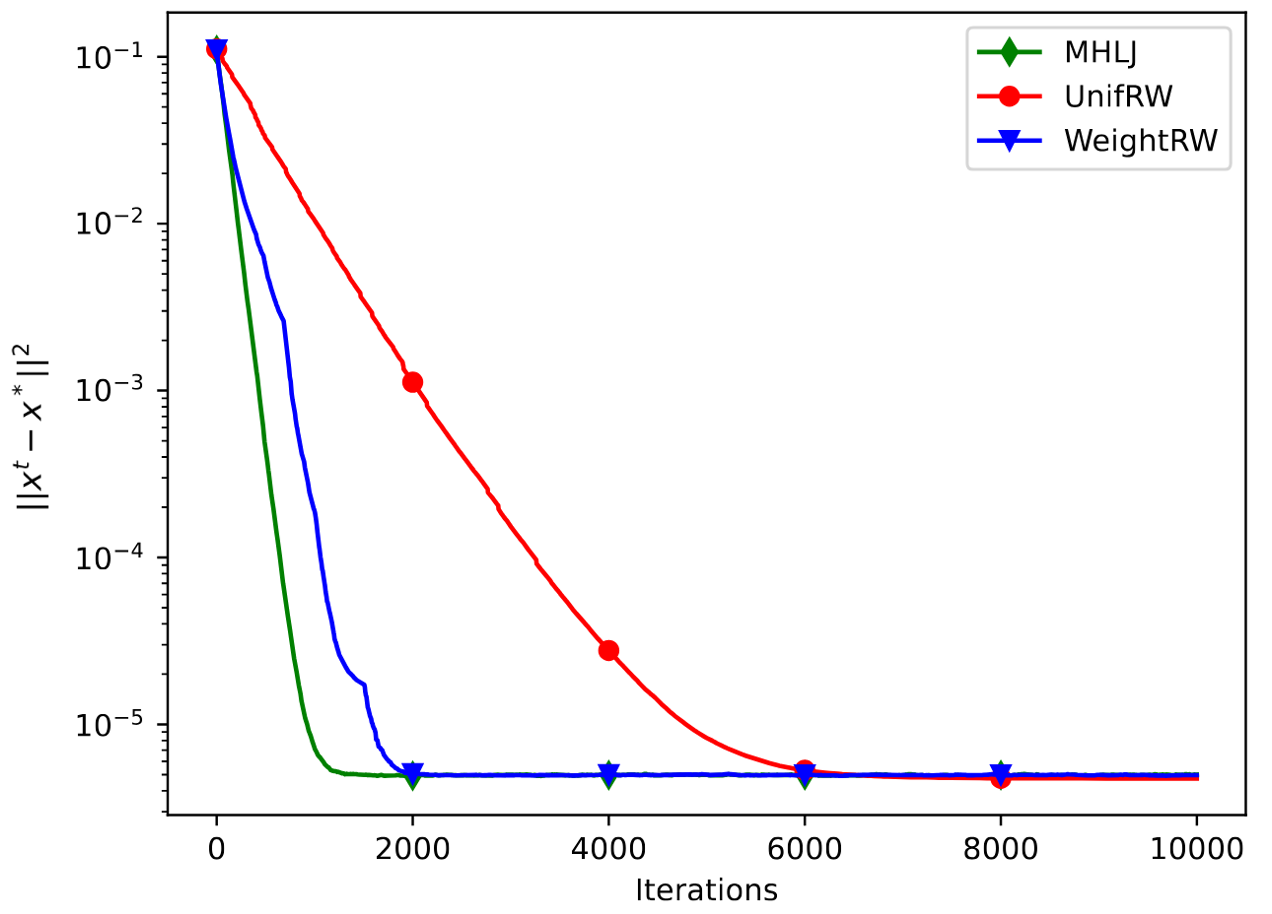}%
\label{hete}}
\caption{Regression model
trained on a synthetic heterogeneous data set over an \er (1000, 0.1) graph. We compare the
UnifRW, WeightRW, and MHLJ.   (a) Homogeneous dataset. (b) Heterogeneous dataset.}
 \label{ER}
\end{figure*}
\begin{figure*}[h]
\centering
\subfloat[]{\includegraphics[width=0.45\textwidth]{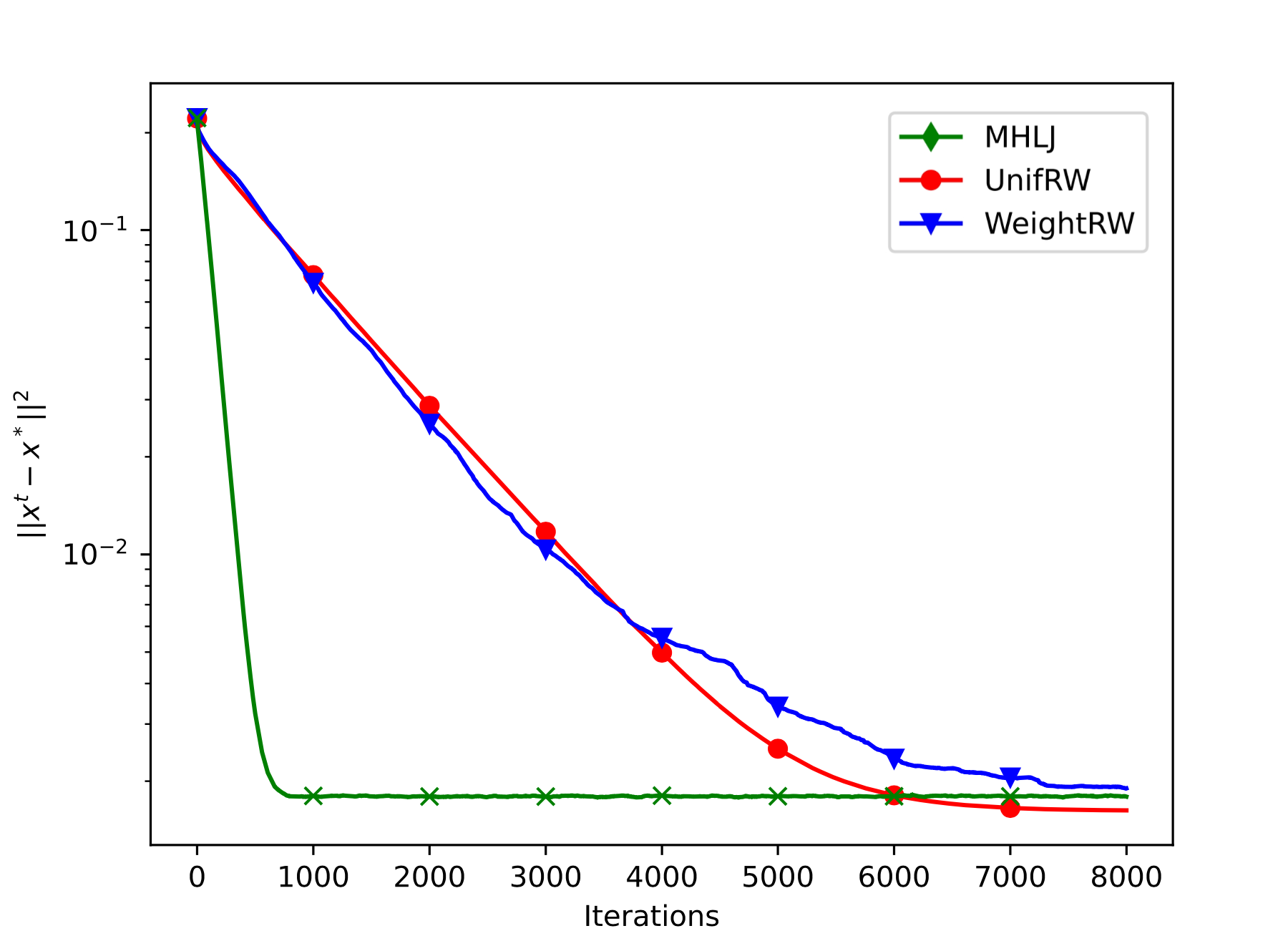}%
\label{grid}}
\hfil
\subfloat[]{\includegraphics[width=0.45\textwidth]{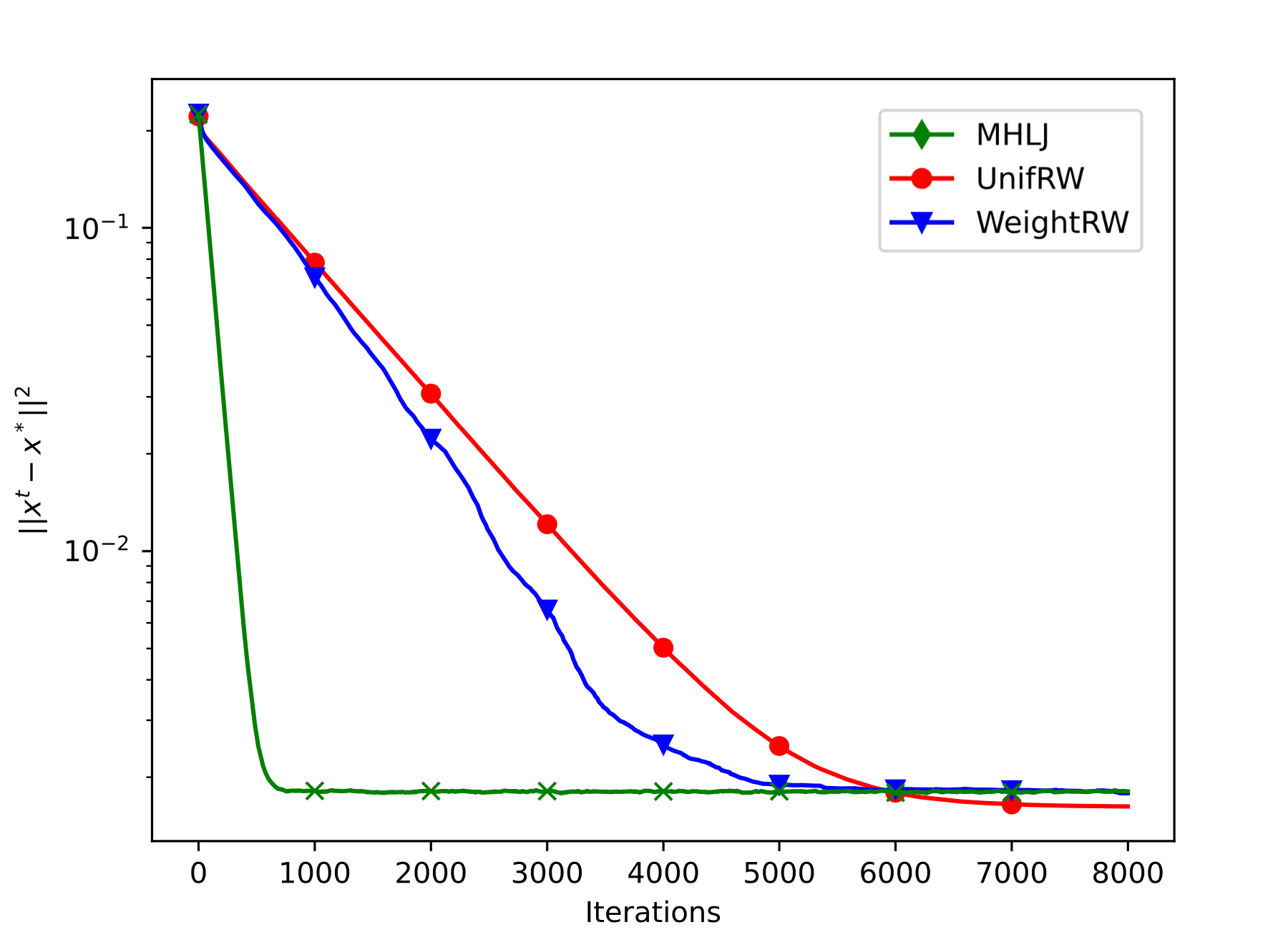}%
\label{wsju}}
\caption{Regression model
trained on a synthetic heterogeneous data set over sparse networks with 1000 nodes. We compare the
UnifRW, WeightRW, and MHLJ. $\sigma^2_H=100$, $\sigma^2_L=1$. (a) 2-d grid. (b) Watts-Strogatz (1000, 4, 0.1) graph.}
\label{ERjump}
\end{figure*}
In this section, we compare the performance of the MHLJ algorithm with the previous random-walk learning algorithms, demonstrate the speedup, and verify the theoretical result from the previous section.  
\subsection{Baseline Algorithms}
We will compare the MHLJ algorithm with three baseline algorithms. Here, we briefly review the algorithms: The uniform random-walk learning (UnifRW), the weighted random-walk learning (WeightRW), and the mixed-weight random walk learning (MixedRW).
\begin{figure*}[h]
\centering
\subfloat[]{\includegraphics[width=0.45\textwidth]{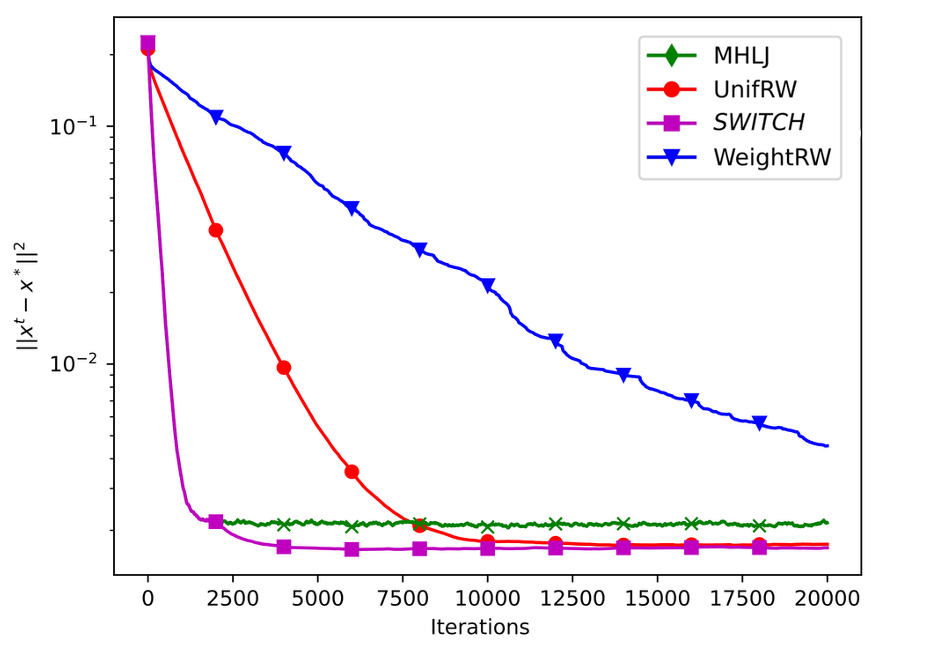}%
 \label{switch}}
\hfil
\subfloat[]{\includegraphics[width=0.44\textwidth]{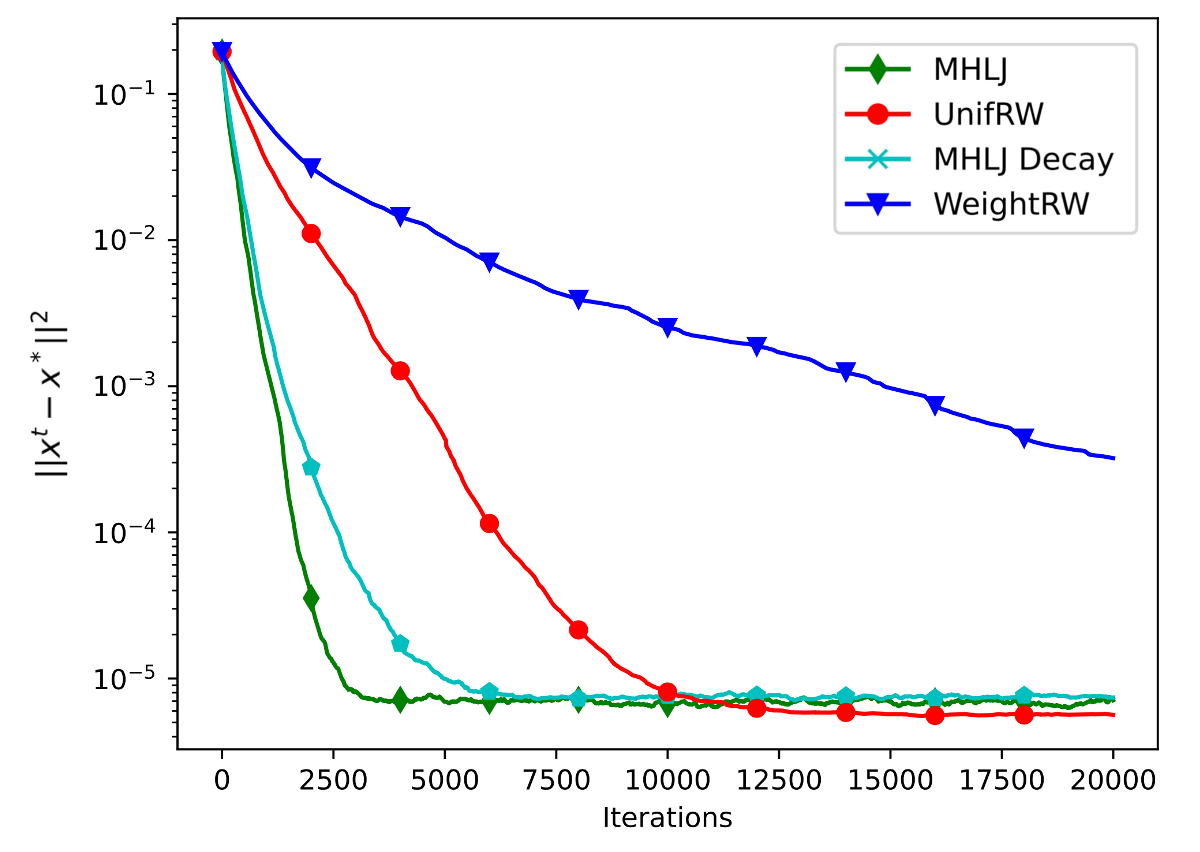}\label{Decay}
 }
\caption{(a) MHLJ switching to the Uniform sampling. (b) MHLJ with a decreasing jump probability.}
 \label{fig:errorgap}
\end{figure*}
\begin{itemize}
    \item UnifRW \cite{johansson2010randomized}: The idea is to push the random walk to visit each node with uniform probability in the long run. To do so, we set the random walk's stationary distribution to be uniform over all nodes. This is done by designing the transition probability in the following way:
    \begin{align*}
        &P(v_t,v_{t+1})=\\ &\begin{cases}
        \frac{1}{\deg(v_t)}
         \min\left\{1, \frac{\deg(v_{t})}{\deg(v_{t+1})}\right\}   & \textit{, } v_{t+1}\in \mathcal{N}_{v_t}\backslash\{v_t\},\\
         1-\sum_{v\in \mathcal{N}_{v_t},v\neq v_t}P(v_t,v)&\textit{, }v_{t+1}=v_t.
        \end{cases}
    \end{align*}
        The model is updated by $\nabla f_{v_t}(x^t)$.
    
    \item WeightRW \cite{ayache2021private} with sampling distribution \eqref{impdist}: The idea is to push the random walk to visit each node with probability \eqref{impdist} in the long run. To do so, we make the stationary distribution of the random walk be \eqref{impdist} over all nodes. This is done by designing the transition probability in the following way:
    \begin{align*}
        &P(v_t,v_{t+1})=\\ &\begin{cases}
        \frac{1}{\deg(v_t)}
         \min\left\{1, \frac{\deg(v_{t})L_{v_{t+1}}}{\deg(v_{t+1})L_{v_t}}\right\}   & \textit{, } v_{t+1}\in \mathcal{N}_{v_t}\backslash\{v_t\}, \\
         1-\sum_{v\in \mathcal{N}_{v_t},v\neq v_t}P(v_t,v)&\textit{, }v_{t+1}=v_t.
        \end{cases}
    \end{align*}
      At iteration $t$, the model is updated by $\frac{1}{L_{v_t}}\nabla f_{v_t}(x^t)$.
    \item MixedRW with parameter $\lambda$: Inspired by \cite{needell2014stochastic}, we also consider a random walk with stationary distribution $\pi_\lambda$ defined below:
    \begin{align*}
        \pi_{\lambda}(v)=\lambda \frac{1}{|V|} + (1-\lambda)\frac{L_v}{\sum_v L_v}, 
    \end{align*}
    similarly, the transition matrix is given by 
    \begin{align*}
         &P_\lambda(v_t,v_{t+1})=\\ &\begin{cases}
        \frac{1}{\deg(v_t)}
         \min\left\{1, \frac{\deg(v_{t})\pi_{\lambda}(v_{t+1})}{\deg(v_{t+1})\pi_\lambda(v_t)}\right\}   & \textit{, } v_{t+1}\in \mathcal{N}_{v_t}\backslash\{v_t\},\\
         1-\sum_{v\in \mathcal{N}_{v_t},v\neq v_t}P(v_t,v)&\textit{, }v_{t+1}=v_t.
        \end{cases}
    \end{align*}

    We see from Corollary~\ref{coro:ratenojump}, the convergence rate of the WeightRW is controlled by $\frac{2\Bar{L}^2}{\eta_{IS} L_{\min }}$. When implementing weighted sampling, the Lipschitz term shrinks from $L_{\max}$ to $\bar{L}$. However, the spectral gap might increase, especially when the graph is sparse, which explains the entrapment problem. So, balancing between UnifRW and WeightRW can beat both.

\end{itemize}
\subsection{Simulation Settings}

\textit{Objective function:}
We test the MHLJ on a linear regression problem. The averaged loss function has the formula:
\begin{align}
    f(x)=\frac{1}{|V|}\sum_{v\in V}(y_v-x^T A_v)^2,
\end{align}
where $\{A_v,y_v\}\in \mathbb{R}^d\times \mathbb{R}$ is the data stored at node $v$, the local loss function at node $v$ is:
\begin{align}
    f_v(x)=\frac{1}{|V|}(y_v-x^TA_v)^2.
\end{align}
The local gradient Lipschitz constant is thus $L_v=2A_v^T A_v$. We thus use the norm of $A_v$ to control the heterogeneity.
We now provide the details of data generation. 

\textit{Data:} 
We generate the homogeneous data set $\{A_v,y_v\}_{v\in V}$ from an isometric Multivariate Gaussian distribution, where the norm of $A_v$ is concentrated around the ball with ratio equal to the standard deviation:
\begin{itemize}
    \item $A_v\overset{\mathrm{i.i.d.}}{\sim} N_{10}(0,\sigma^2\mathbb{I}_{10})$.
    \item $y_v=A_v^T x+\epsilon$, where $\epsilon\overset{\mathrm{i.i.d.}}{\sim}N(0,1)$.
\end{itemize}
The heterogeneous data set $\{A_v,y_v\}_{v\in V}$ is generated from a mixed Multivariate Gaussian distribution, i.e., with a large probability, we sample the data from a Gaussian with a small standard deviation, and with a small probability, the standard deviation is large: 
\begin{itemize}
    \item $A_v|\sigma^2\overset{\mathrm{i.i.d.}}{\sim} N_{10}(0,\sigma^2\mathbb{I}_{10})$, where $\sigma^2$ takes value $\sigma^2_L$ with probability $p=0.998$ and $\sigma^2_H$ with probability $p=0.002$.
    \item $y_v=A_v^T x+\epsilon$, where $\epsilon\overset{\mathrm{i.i.d.}}{\sim}N(0,1)$.
\end{itemize}
For each node $v$, we assign one data point $(X_v, y_v)$.

\textit{Networks:}
We consider two families here: the \er\ random graph, which represents the well-connected random graph with graph parameters $(n=1000,p=0.1)$. And three types of poorly connected graphs: the ring graphs, the 2D grids, and the Watts-Strogatz~$(1000, 4, 0.1)$ random graph.

\subsection{Simulation Results}

\subsubsection{\er\ random graph}
We show in this subsection the performance of UnifRW, WeightRW, and MHLJ on an \er $(1000, 0.1)$ random graph. The simulation result is described in Figure~\ref{ER}. We see from the simulation:
\begin{itemize}
    \item The three algorithms perform similarly when the data are homogeneous.
    \item When the data are heterogeneous, WeightRW beats UnifRW, and MHLJ is even faster.
\end{itemize}
These observations show MHLJ performs at least as well as UnifRW; when the data are heterogeneous, MHLJ exhibits better convergence behavior than WeightRW. It benefits from weighted sampling and can move faster along the graph, thus learning faster from the data.

\subsubsection{Performance of MHLJ on other sparse graph families}
In addition to the ring graph, we also test MHLJ on the Watts-Strogatz graph with $ (1000, 4, 0.1)$ and a 2D grid network with heterogeneous data. The average degree for both families is $4$. We observe that in both cases, the entrapment problem occurs, and MHLJ significantly speeds up convergence. 


\subsubsection{Eliminating the error gap}
 We now discuss two ways to eliminate the error gap. One practical approach is to switch to uniform sampling via \mh once the loss function stabilizes; see Figure~\ref{switch} for the simulation result. We see that switching the sampling and updating strategy to uniform sampling eliminates the error gap without sacrificing convergence speed. To detect when to switch, the model should store a recent history of gradients. The algorithm should switch to uniform sampling at specific steps when the sum is close to 0. The other method is to gradually reduce the jump probability. The simulation result is shown in Figure~\ref{Decay}. However, we see that this method doesn't recover the error gap. This is because when $\sigma^2_*$ is insignificant, the error gap is tiny, and the spectral gap term in the first term in \eqref{thmhlj} dominates the error. 
 \subsubsection{Compare with MixRW}
We compare the performance of using a mixed distribution with MHLJ. From Figure~\ref{fig:mixeddist}, we see that the mixed distribution can also significantly improve the convergence rate, but it is still slower than MHLJ. On the other hand, the best $\lambda$ depends largely on the dataset and the graph structure, and finding a good $\lambda$ can be hard due to the non-convexity of $\eta_\lambda$ as a function of $\pi_\lambda$ (the construction of $P_\lambda$ with respect to $\pi_\lambda$ is via the Metropolis-Hastings algorithm), which makes the mixed distribution unrealistic in real applications.
 \begin{figure}[t]
\centering
\subfloat[]{\includegraphics[width=0.45\textwidth]{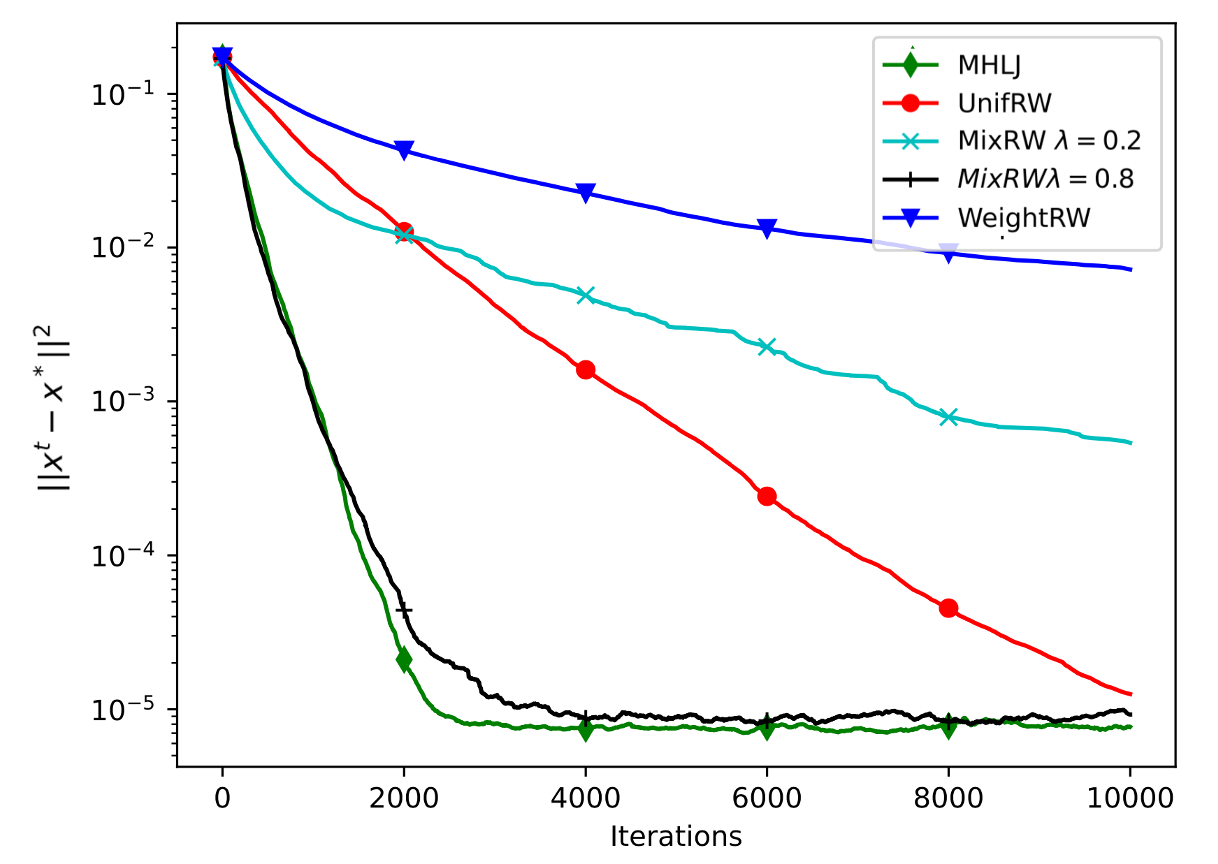}%
\label{fig:mixdecay}}

\caption{Performance of mixed distribution and MHLJ on ring graph.}
 \label{fig:mixeddist}
\end{figure}

\section{Proofs}
We prove Theorem~\ref{thmnojump} and Theorem~\ref{thmhlj}  in this section.  
The main difficulty in proving the convergence of the random-walk SGD is due to the one-step conditional biases, which are caused by the graph topology:
\begin{align}
    \E{\nabla f_{v_t}(x^t)\mid v_{t-1}}\neq \nabla f(x^t).
\end{align}
We tackle this by constructing a ``good'' auxiliary sequence $\{y^t\}_{t=0}^T$ that depends on the optimum $x^*$ and the trajectory of the random walk. This idea is originally from \cite{even2023stochastic}, although, in this paper, we use a different proof that can be generalized to the biased stochastic gradient case. With the help of $\{y^t\}$, we can quantify the accumulated sampling bias due to dependency and upper-bound it by the convergence rate of the empirical sample from an ergodic Markov chain.

For the convergence of MHLJ, we have another difficulty, i.e., the update is not ergodically unbiased, which leads to:
\begin{align*}
    \mathbb{E}_{v\sim \pi}[\nabla f_v(x^*)]\neq \nabla f(x^*) = 0,
\end{align*}
this means we need to be more careful when using the first-order optimal condition.          
\subsection{Intermediate Lemmas }
We consider the convergence of weighted random-walk SGD with sampling distribution $\nu$ and SGD update given by:
\begin{align}
    x^t=x^{t-1}-\gamma w(v_{t-1})\nabla f_{v_{t-1}}(x^{t-1}).\label{xseq}
\end{align}
Here, we allow $\nu(v)\neq \frac{1}{w(v)}$ to adapt to the ergodically biased stochastic gradient method.
We construct an auxiliary sequence $\{y^t\}_{t=0}^T$ by:
\begin{align}
    y^{t}&=y^{t-1}-\gamma w(v_{t-1})\nabla f_{v_{t-1}}(x^*),\label{auxiliary}
\end{align}
where $x^*$ is the optimum of \eqref{obj}.
Note that, given a value of any $y^s,s\in[T]$, the sequence $\{y^t\}_{t=0}^T$ is determined by $\{v_t\}_{t=0}^T$. 
The following Lemma controls the distance between the two sequences. This result is modified from \cite[Lemma 9]{even2023stochastic}.
\begin{lem}
    If $x^t$ is generated from \eqref{xseq}. Denote $L^*=\max_{v}\frac{L_v}{w(v)}$, if $\gamma\leq L^*$, then for any $\{y^t\}_{t=0}^T$ we have 
    \begin{align*}
    \norm{x^{t+1}-y^{t+1}}^2&\leq (1-\gamma \mu)\norm{x^t-y^t}^2+\gamma L^*\norm{y^t-x^*}^2.
\end{align*}
\label{lemdistance}
\end{lem}

Lemma~\ref{lemdistance} doesn't involve any randomness, i.e., the relation holds for any given trajectory regardless of the underlying transition probability $P$. In our case, we have: for uniform sampling, i.e., when $w(v)\equiv 1/|V|$, $L^*=L_{\max}$; for weighted sampling, i.e., when $w(v)=L_v/\Bar{L}$, $L^*=\Bar{L}$.

By writing $y^t=x^*+\gamma\sum_{t\leq s\leq T}w(v_s)\nabla f_{v_s}(x^*)$, we can relates $\norm{x^T-x^*}^2$ with $\left\Vert\sum_{i=s}^tw(v)\nabla f_{v_i}(x^*)\right\Vert^2$. Our main contribution is the following Lemma on bounding the accumulated $L_2$-norm of $w(v_t)\nabla f_{v_t}(x^*)$. 
\begin{lem}
Let $P$ be the transition matrix of the random walk with stationary distribution $\nu$, then starting from $\nu$, for $1\leq s\leq t\leq T$, we have 
    \begin{align*}
        &\E{\left\Vert\sum_{i=s}^tw(v)\nabla f_{v_i}(x^*)\right\Vert^2}\leq\\ &(t-s)\frac{w_{\max}}{ \eta }\sigma_{\max}^2+2(t-s)^2 \norm{\pi-\nu}_{TV}^2\sigma_{\max}^2 w_{\max}^2,
    \end{align*}
    where $\eta$ is the spectral gap of chain $P$, $\pi=\frac{w(v)^{-1}}{\sum w(v)^{-1}}$, $w_{max} =\max_v w_v$. 
    
    When $\frac{w(v)^{-1}}{\sum w(v)^{-1}}=\nu(v)$, and $P$ is reversible, we have 
    \begin{align}
        \E{\left\Vert\sum_{i=s}^tw(v)\nabla f_{v_i}(x^*)\right\Vert^2}\leq (t-s)\frac{w_{\max}}{ \eta }\sigma_{*}^2.\label{unbiased}
    \end{align}
    \label{lemempsum}
\end{lem}
\begin{proof}[Proof of Lemma~\ref{lemempsum}]

\begin{align*}
    &\frac{1}{(t-s)^2}\E{\left\Vert\sum_{i=s}^tw(v_i)\nabla f_{v_i}(x^*)\right\Vert^2}\\
     =&\E{\left\Vert\frac{1}{t-s}\sum_{i=s}^t w(v_i)\nabla f_{v_i}(x^*)\right\Vert^2}\\
    \overset{(a)}{\leq}&2\E{\left\Vert\tfrac{\sum_{i=s}^t w(v_i)\nabla f_{v_i}(x^*)}{t-s}
    -\mathbb{E}_{\nu}\lbrack w(v_i)\nabla f_{v_i}(x^*)\rbrack\right\Vert^2}\\
    &+2\left\Vert\mathbb{E}_{\nu}\lbrack w(v_i)\nabla f_{v_i}(x^*)\rbrack-\mathbb{E}_{\pi}\lbrack w(v_i)\nabla f_{v_i}(x^*)\rbrack\right\Vert^2\\
    \overset{(b)}{\leq}& \tfrac{2\Tilde{\sigma}_{\max}^2}{\eta (t-s)}+2\left\Vert\mathbb{E}_{\nu}\lbrack w(v_i)\nabla f_{v_i}(x^*)\rbrack-\mathbb{E}_{\pi}\lbrack w(v_i)\nabla f_{v_i}(x^*)\rbrack\right\Vert^2\\
    \leq& \tfrac{2w_{\max}\sigma^2_{\max}}{\eta(t-s)} +2\norm{\nu-\pi}_{TV}^2\sigma_{\max}^2 w_{\max}^2,
    \end{align*}
    where $\Tilde{\sigma}^2_{\max}\leq w_{\max}\sigma^2_{\max}$, (a) follows from $\norm{x+y}^2\leq 2\norm{x}^2+2\norm{y}^2$ and the first order optimal condition 
    \begin{align*}
&\mathbb{E}_{v\sim\pi}\left[w(v)\nabla f_v(x^*)\right ]=\sum_v \pi(v)w(v)\nabla f_v(x^*)\\=&\tfrac{1}{\sum_vw(v)^{-1}}\sum_v \nabla f_v(x^*)=0.
    \end{align*} 
    And (b) follows from Naor et al. \cite[Theorem 1.2]{naor2020concentration}\footnote{Suppose that $\mathbf{X}=\left\{X_t\right\}_{t=1}^{\infty}$ is a stationary Markov chain whose state space is $[N]$ and with $\eta<\infty$. Then, every $f:[N] \rightarrow \mathbb{R}$ satisfies the following inequality for every $n \in \mathbb{N}$:
    \begin{align*}
&\mathbb{E}\left[\left|\tfrac{f\left(X_1\right)+\cdots+f\left(X_n\right)}{n}-\mathbb{E}\left[f\left(X_1\right)\right]\right|^2\right]   \lesssim \tfrac{2}{\eta n}\cdot (\max \{|f(1)|, \ldots,|f(N)|\})^2 .
    \end{align*}
 }, the last inequality is due to:
 
    \begin{align*}
&\norm{\mathbb{E}_{v\sim\pi}[w(v)\nabla f_{v}(x^*)]-\mathbb{E}_{v\sim\nu}[w(v)\nabla f_{v}(x^*)]}^2\\&\leq \norm{\pi-\nu}_{TV}^2\cdot \max_v \norm{\nabla f_v(x^*)}^2 w_{\max}^2.
    \end{align*}
    We have the desired result by multiplying $(t-s)^2$ on both sides.

    When $\tfrac{w(v)^{-1}}{\sum w(v)^{-1}}=\nu(v)$, we have $\mathbb{E}_{v\sim v}[w(v)\nabla f_v(x^*)]=0,$ if we have further $P$ being reversible, then by Levin and Peres \cite[Lemma 12.22]{levin2017markov}, (b) becomes 
    \begin{align*}\small
       \frac{1}{(t-s)^2}\E{\left\Vert\sum_{i=s}^tw(v_i)\nabla f_{v_i}(x^*)\right\Vert^2} {\leq} \frac{2w_{\max}\sigma_{*}^2}{ \eta (t-s)}.
    \end{align*}

\end{proof}
\subsection{Proof of Theorem~\ref{thmnojump}} 
\begin{proof}
Consider the weighted sampling, where $\pi(v)\propto\frac{L_v}{\Bar{L}}$ and $w(v)=\frac{\Bar{L}}{L_v}$. Set $y^T=x^*$, from Lemma~\ref{lemdistance}, we have
    \begin{align*}
    &\norm{x^T-y^T}^2\leq\\& (1-\gamma\mu)^T\norm{x^0-y^0}^2+\gamma \Bar{L}\sum_{t\leq T}(1-\gamma\mu)^{T-t}\norm{y^t-x^*}^2.
    \end{align*}
    Note that
    \begin{align*}
        y^0&=x^*+\gamma\sum_{t\leq T}w(v_t)\nabla f_{v_t}(x^*),\\
        y^t&=x^*+\gamma\sum_{t\leq s\leq T}w(v_s)\nabla f_{v_s}(x^*).
    \end{align*}
    We can upper bound:
    \begin{align}
        &\E{\norm{x^T-x^*}^2}
        \leq 2(1-\gamma\mu)^T\norm{x^0-x^*}^2\notag\\+&3\gamma^3\Bar{L}\sum_{t\leq T}(1-\gamma\mu)^{T-t}\E{\left\Vert\sum_{t\leq s\leq T}w(v_s)\nabla f_{v_s}(x^*)\right\Vert^2}.
        \label{start}
        \end{align}
        The third term can be upper bounded by Lemma~\ref{lemempsum}; since the Metropolis-Hastings chain is reversible, we have
        \begin{align*}
            &\E{\norm{x^T-x^*}^2}            \leq 2(1-\gamma\mu)^T\norm{x^0-x^*}^2\\&+\eta^{-1}\gamma^3\Bar{L}\sum_{t\leq T}(1-\gamma\mu)^{T-t}(T-t)\tfrac{\Bar{L}}{L_{\min}}\sigma_{*}^2.
        \end{align*}
        Finally, use the numerical inequality $\sum_{t\leq T}(1-x)^t t\leq 1/x^2$, we have
        \begin{align*}
            \E{\norm{x^T-x^*}^2}\leq 2(1-\gamma\mu)^T\norm{x^0-x^*}^2+\tfrac{\gamma\Bar{L}^2\sigma_{*}^2}{L_{\min}\eta\mu^2}. \end{align*}
 \end{proof}

\subsection{Proof of Theorem~\ref{thmhlj}}
\begin{proof}[Proof of Theorem~\ref{thmhlj}]
Start from \eqref{start}:
    \begin{align*}
        &\E{\norm{x^T-x^*}^2}
          \leq 2(1-\gamma\mu)^T\norm{x^0-x^*}^2\\&+3\gamma^3\Bar{L}\sum_{t\leq T}(1-\gamma\mu)^{T-t}\E{\left\Vert\sum_{t\leq s\leq T}w(v_s)\nabla f_{v_s}(x^*)\right\Vert^2}.
            \end{align*}
       Upper bounding the third term is then upper bounded by Lemma~\ref{lemempsum}, we have
        \begin{align*}
            &\E{\norm{x^T-x^*}^2}\leq 2(1-\gamma\mu)^T\norm{x^0-x^*}^2\\&+\eta^{-1}\gamma^3\Bar{L}\sum_{t\leq T}(1-\gamma\mu)^{T-t}(T-s)\tfrac{\Bar{L}}{L_{\min}}\sigma_{\max}^2\\
        &+ \norm{\pi-\nu}_{TV}^2\sigma_{\max}^2\gamma^3\tfrac{\Bar{L}^3}{L_{\min}^2}\sum_{t\leq T}(1-\gamma\mu)^{T-t}(T-t)^2\\
        \overset{(a)}{\leq}&2(1-\gamma\mu)^T\norm{x^0-x^*}^2\\&+\eta^{-1}\gamma^3\Bar{L}\sum_{t\leq T}(1-\gamma\mu)^{T-t}(T-s)\tfrac{\Bar{L}}{L_{\min}}\sigma_{\max}^2.
        \end{align*}
        Finally, use the numerical inequality $\sum_{t\leq T}(1-x)^t t\leq 1/x^2$ and $\sum_{t\leq T}(1-x)^t t^2\leq \frac{2}{x^3}$, we have
        \begin{align*}
            &\E{\norm{x^T-x^*}^2}\leq 2(1-\gamma\mu)^T\norm{x^0-x^*}^2
            +\tfrac{\gamma\Bar{L}^2\sigma_{\max}^2}{L_{\min}\eta\mu^2}\\&+\tfrac{p_J^2\norm{\pi-\Tilde{\pi}}_{TV}^2\sigma_{\max}^2 \Bar{L}^3}{\mu^3 L_{\min}^2}. \end{align*}
\end{proof}
\subsection{Proof of Corollary~\ref{coro:ratenojump}}
\begin{proof}
    Since $\gamma\leq \frac{\epsilon L_{\min}\eta_{IS}\mu^2}{2\Bar{L}\sigma_*^2}$, we have the second term in \eqref{rateis} being smaller than $\epsilon/2$. Now asking that
    \begin{align*}
        2(1-\gamma\mu)^T\norm{x^0-x^*}^2\leq\epsilon/2,
    \end{align*}
    we have 
    \begin{align*}
        T\log\left(1-\tfrac{\epsilon L_{\min}\eta_{IS}\mu^3}{2\Bar{L}\sigma_*^2}\right)\leq\log(\epsilon/4e_0),
    \end{align*}
    where $e_0 = \norm{x^0-x^*}^2$, by the numerical inequality $-1/\log(1-x)\leq 1/x$, we have:
    \begin{align*}
        T\geq \log(\tfrac{4e_0}{\epsilon})\tfrac{2\Bar{L}\sigma_*^2}{\epsilon L_{\min}\eta_{IS}\mu^3}.
    \end{align*}
\end{proof}
\nocite{seneta1988perturbation,nedic2018network,gholami2024digest,gholami2024improved}
\bibliographystyle{IEEEtran}
\bibliography{ref}

\end{document}